\documentclass[letterpaper,10pt,conference]{ieeeconf}
\pdfoutput=1
\usepackage{cite}
\usepackage{graphicx}
\usepackage{subcaption}
\usepackage[cmex10]{amsmath}
\usepackage{algorithm}
\usepackage{algorithmic}
\usepackage{xcolor}
\usepackage{array}
\usepackage{url}
\usepackage{bbm}
\usepackage{todonotes}

\usepackage{epsfig}
\usepackage{amsfonts,amssymb,multirow,bigstrut,booktabs,ctable,latexsym}
\usepackage{mathtools}
\usepackage[mathscr]{eucal}
\usepackage{verbatim}

\usepackage{amsthm}
\usepackage{algorithm}
\usepackage[normalem]{ulem}
\newtheorem{definition}{Definition}

\newtheorem{example}{Example}
\newtheorem{proposition}{Proposition}
\newtheorem{remark}{Remark}

\IEEEoverridecommandlockouts

	\title{\bf Reinforcement Learning Beyond Expectation}
	
	\author{Bhaskar Ramasubramanian$^{1}$, Luyao Niu$^{2}$, Andrew Clark$^{2}$, and Radha Poovendran$^{1}$%
		\thanks{$^{1}$Network Security Lab, Department of Electrical and Computer Engineering, 
			University of Washington, Seattle, WA 98195, USA. \newline
			{\tt\small \{bhaskarr, rp3\}@uw.edu}}
\thanks{$^2$Department of Electrical and Computer Engineering, Worcester Polytechnic Institute, Worcester, MA 01609 USA.
			{\tt\{lniu,aclark\}@wpi.edu}}
	}

\begin{document}	
	\maketitle
	
	
\begin{abstract}
The inputs and preferences of human users are important considerations in situations where these users interact with autonomous cyber or cyber-physical systems. 
In these scenarios, one is often interested in aligning behaviors of the system with the preferences of one or more human users. 
Cumulative prospect theory (CPT) is a paradigm that has been empirically shown to model a tendency of humans to view gains and losses differently. 
In this paper, we consider a setting where an autonomous agent has to learn behaviors in an unknown environment. 
In traditional reinforcement learning, these behaviors are learned through repeated interactions with the environment by optimizing an expected utility. 
In order to endow the agent with the ability to closely mimic the behavior of human users, we optimize a CPT-based cost. 
We introduce the notion of the CPT-value of an action taken in a state, and establish the convergence of an iterative dynamic programming-based approach to estimate this quantity. 
We develop two algorithms to enable agents to learn policies to optimize the CPT-value, and evaluate these algorithms in environments where a target state has to be reached while avoiding obstacles. 
We demonstrate that behaviors of the agent learned using these algorithms are better aligned with that of a human user who might be placed in the same environment, and is significantly improved over a baseline that optimizes an expected utility. 
\end{abstract}

\section{Introduction}\label{Introduction}

Many problems in cyber and cyber-physical systems 
involve sequential decision making under uncertainty to accomplish desired goals. 
These systems are dynamic in nature, and 
actions are chosen in order 
to maximize an accumulated reward or minimize a total cost. 
Paradigms to accomplish these objectives include reinforcement learning (RL) \cite{sutton2018reinforcement} and optimal control \cite{bertsekas2017dynamic}. 
The system is typically represented as a Markov decision process (MDP) \cite{puterman2014markov}. 
Transitions between successive states of the system is a probabilistic outcome that depends on actions of the decision maker or agent. 
Decision making under uncertainty can then be expressed in terms of maximizing an expected utility (accumulated reward or negative of accumulated cost). 
Such an agent is said to be \emph{risk-neutral}. 
These frameworks have been successfully implemented in multiple domains, including robotics, games, power systems, and mobile networks \cite{hafner2011reinforcement, mnih2015human, silver2016mastering, zhang2019deep, sadigh2016planning, yan2018data, you2019advanced}. 

An alternative is taking a \emph{risk-sensitive} approach to decision making. 
A \emph{risk-averse} agent might be willing to forego a higher expected utility if they were to have a higher certainty of an option with a lower utility. 
Conversely, an agent can be \emph{risk-seeking} if they prefer less certain options that have a higher utility. 
Risk-neutral and risk-averse agents are considered to be rational, while a risk-seeking agent is considered irrational \cite{gollier2001economics, gilboa2009theory}. 
The incorporation of risk into the behavior of a decision maker has been typically carried out by computing the expectation of a transformation of the utility obtained by the agent \cite{shen2013risk, shen2014risk}. 

As an illustrative example, consider a navigation problem where there are two possible routes from a source to a destination. 
The first route is faster, on average, but there is a chance of encountering a delay which can significantly increase the total time taken. 
The second route is slightly slower, on average, but the chance of encountering a delay is smaller. 
A risk-neutral agent may opt to take the first route to minimize the average travel time. 
However, a risk-sensitive agent might have a preference for not encountering delays, and thus would opt to take the second route. 

The inputs and preferences of human users is playing an increasingly important role in 
scenarios where actions of human users and possibly autonomous complex systems influence each other in a shared environment. 
In these situations, one is interested in aligning behaviors of the system with preferences of one or more human users. 
Human users often exhibit behaviors that may not be considered entirely rational due to various cognitive and emotional biases. 
Moreover, they can exhibit both risk-seeking and risk-averse behaviors in different situations. 
In these situations, it has been observed that expected utility-based frameworks are not adequate to describe human decision making, 
since humans might have a different perception of both, the same utility and the same probabilistic outcome as a consequence of their decisions \cite{kahneman1979prospect}. 

Cumulative prospect theory (CPT), introduced in \cite{tversky1992advances}, has been empirically shown to capture preferences of humans for certain outcomes over certain others. 
The key insight guiding CPT is that humans often evaluate potential gains and losses using heuristics, and take decisions based on these. 
In particular, CPT is able to address a tendency of humans to: i) be risk-averse with gains and risk-seeking with losses; ii) distort extremely high and low probability events. 
To model the former, CPT uses a non-linear utility function to transform outcomes. 
The latter is addressed by using a non-linear weighting function to distort probabilities in the cumulative distribution function. 
Moreover, the utility and weighting functions corresponding to gains and losses can be different, indicating that gains and losses are often interpreted in different ways by a human. 

In this paper, we develop a framework for CPT-based decision-making in settings where a model of the system of interest is not available. 
An agent in such a scenario will have to learn behaviors through minimizing a cost signal revealed through repeated interactions with the system. 
We seek to endow these agents with the ability to take decisions that are aligned with the preferences of humans. 
To accomplish this, we optimize the sum of CPT-value period costs using a dynamic programming based approach. 
We develop an iterative procedure to learn policies and establish conditions to ensure convergence of the procedure. 
We demonstrate that the behavior of agents using policies learned by minimizing a CPT-based cost mimic those of a human user more closely than in the case when policies are learned when an expected accumulated cost is minimized. 

We make the following contributions in this paper:
\begin{itemize}
\item We define the \emph{CPT-value of a state-action pair} called \emph{CPT-Q} in order to develop a method to optimize a CPT-based cost using reinforcement learning. 
\item We introduce a \emph{CPT-Q iteration} to estimate CPT-Q for each state-action pair, and demonstrate its convergence. 
\item We develop two algorithms, CPT-SARSA and CPT-Actor-Critic to estimate CPT-Q. 
\item We evaluate the above algorithms in environments where a target state has to be reached while avoiding obstacles. We demonstrate that behaviors of an RL agent when following policies learned using CPT-SARSA and CPT-Actor-Critic are aligned with that of a human user who might be placed in the same environment. 
\end{itemize}

The remainder of this paper is organized as follows. 
Section \ref{RelatedWork} gives an overview of related literature. 
Section \ref{Preliminaries} gives background on MDPs, RL, and risk measures. 
The CPT-value of a random variable is defined in Section \ref{CPTDefn}. 
We introduce the reinforcement learning framework that uses CPT and prove our main results in Section \ref{CPTRLResults}. 
Section \ref{CPTRLAlgos} details the developments of two algorithms to solve the CPT-RL problem. 
We present an evaluation of our approach in Section \ref{Simulation}, and Section \ref{Conclusion} concludes the paper. 

\section{Related Work}\label{RelatedWork}

Frameworks that incorporate risk-sensitivity in reinforcement learning and optimal control 
typically replace the utility (say, $U$) with a function of the utility (say, $U'$). 
Some examples include a mean-variance tradeoff \cite{markowitz1952portfolio, tamar2012policy, mannor2013algorithmic}, exponential function of the utility \cite{howard1972risk, whittle1990risk, borkar2002q}, and conditional value at risk (CVaR) \cite{rockafellar2002conditional}. 
The CVaR corresponds to the average value of the cost conditioned on the event that the cost takes sufficiently large values. 
CVaR has been shown to have a strong theoretical justification for its use, and optimizing a CVaR-based cost will ensure sensitivity of actions to rare high-consequence outcomes. 
The risk-sensitivity has also been represented as a constraint that needs to be satisfied while an expected utility $U$ is maximized. 
We refer the reader to \cite{prashanth2018risk} for an exposition on these methods. 
A common theme among the approaches outlined above is that the objective in each case is to maximize an expectation over $U'$.  

Optimization of a CPT-based cost in an online setting was studied in \cite{prashanth2016cumulative, jie2018stochastic}, where the authors optimized the CPT-value of the return of a policy. 
A different approach was adopted in \cite{lin2013dynamic, lin2018probabilistically} where the authors optimized a sum of CPT-value period costs using a dynamic programming based approach when a model of the system was available. 
A computational approach to verifying reachability properties in MDPs with CPT-based objectives was presented in \cite{cubuktepe2018verification}. 
The authors of this work  approximated the weighting function in the CPT-value as a difference of convex functions and used the convex-concave procedure \cite{lipp2016variations} to compute policies. 

We distinguish the contributions of this paper in comparison to prior work in two ways. 
Different from work on risk-sensitive control that aims to minimize an expected cost subject to a threshold-based risk constraint like CVaR, in this paper, we seek to optimize an objective stated in terms of a CPT-value. 
We also do not assume that a model of the environment is available. 
Instead, the agent will have to learn policies through repeated interactions with the environment. 
This will inform our development of reinforcement learning algorithms in order to learn optimal policies when the agent seeks to minimize a CPT-based cost. 

\section{Preliminaries}\label{Preliminaries}

\subsection{MDPs and RL}

Let $(\Omega, \mathcal{F}, \mathcal{P})$ denote a probability space, where $\Omega$ is a sample space, $\mathcal{F}$ is a $\sigma-$algebra of subsets of $\Omega$, and $\mathcal{P}$ is a probability measure on $\mathcal{F}$. 
A random variable is a map $Y: \Omega \rightarrow \mathbb{R}$. 
We assume that the environment of the RL agent is described by a Markov decision process (MDP) \cite{puterman2014markov}. 

\begin{definition}
An MDP is a tuple $\mathcal{M}:= (S, A, \rho_0, \mathbb{P}, c, \gamma)$, where $S$ is a finite set of states, $A$ is a finite set of actions, and $\rho_0$ is a probability distribution over the initial states. 
$\mathbb{P}(s'|s,a)$ is the probability of transiting to state $s'$ when action $a$ is taken in state $s$. 
$c: S \times A \rightarrow \mathbb{R}$ is the cost incurred by the agent when it takes action $a$ in state $s$. 
$\gamma \in (0,1]$ is a discounting factor which indicates that at any time, we care more about the immediate cost than costs that may be incurred in the future. 
\end{definition}

An RL agent typically does not have knowledge of the transition function $\mathbb{P}$. 
Instead, it incurs a cost $c$ for each action that it takes. 
We assume that $c$ is a random variable such that $|c| < \infty$. 
Through repeated interactions with the environment, the agent seeks to learn a policy $\pi$ in order to minimize an objective $\mathbb{E}_\pi[\sum_t \gamma^t c(s_t,a_t)]$ \cite{sutton2018reinforcement}. 
A \emph{policy} is a probability distribution over the set of actions at a given state, and is denoted $\pi(\cdot|s)$.

\subsection{Risk Measures}

For a set of random variables on $\Omega$, denoted $\mathcal{Y}$, a \emph{risk measure} or \emph{risk metric} is a map $\rho: \mathcal{Y} \rightarrow \mathbb{R}$ \cite{majumdar2020should}. 

\begin{definition}
A risk metric is \textbf{coherent} if it satisfies the following properties for all $Y, Y_1, Y_2 \in \mathcal{Y}, d \in \mathbb{R}, m \in \mathbb{R}_{\geq 0}$: 
\begin{enumerate}
\item Monotonicity: $Y_1(\omega) \leq Y_2(\omega)$ for all $\omega \in \Omega$ $\Rightarrow$ $\rho(Y_1) \leq \rho(Y_2)$; 
\item Translation invariance: $\rho(Y+d) = \rho(Y) + d$; 
\item Positive homogeneity: $\rho(mY) = m . \rho(Y)$; 
\item Subadditivity: $\rho(Y_1+Y_2) \leq \rho(Y_1) + \rho(Y_2)$. 
\end{enumerate}
\end{definition}

The last two properties together ensure that a coherent risk metric will also be convex. 

\begin{example}\label{EgRiskMetrics}
Examples of risk metrics include: 
\begin{enumerate}
\item \textbf{Expectation} of a random variable, $\mathbb{E}[Y]$;
\item \textbf{Value at Risk} at level $\alpha \in (0,1)$: $VaR_\alpha (Y):= \inf\{y | \mathbb{P}[Y \leq y] \geq \alpha\}$; 
\item \textbf{Conditional Value at Risk}: $CVaR_\alpha (Y)$ is a conditional mean over the tail distribution, as delineated by $VaR_\alpha$. Thus, $CVaR_\alpha (Y) = \mathbb{E}[Y | Y \geq VaR_\alpha (Y)]$. Alternatively, with $(\cdot)^+:=\max(\cdot,0)$, we can write: 
\begin{align}
CVaR_\alpha (Y)&:=\min_{s \in \mathbb{R}}\big[s+\frac{1}{1-\alpha} \mathbb{E}[(Y-s)^+]\big]. \label{CVaRalpha}
\end{align}
\end{enumerate}
\end{example}

Risk metrics such as $VaR_\alpha$ and $CVaR_\alpha$ quantify the severity of events that occur in the tail of a probability distribution. 
$CVaR_\alpha$ is an example of a coherent risk metric. 

In this paper, we are interested in determining policies to optimize objectives expressed in terms of more general risk metrics. 
The risk metric that we adopt in this paper is informed from cumulative prospect theory \cite{tversky1992advances}, and is not coherent. 

\section{Cumulative Prospect Theory}\label{CPTDefn}

Human players or operators have been known to demonstrate a preference to play safe with gains and take risks with losses. 
Further, they tend to \emph{deflate} high probability events, and \emph{inflate} low probability events. 
This is demonstrated in the following example. 

\begin{example}
Consider a game where one can either earn $\$500$ with probability (w.p.) $1$ or earn $\$5000$ w.p. $0.1$ and nothing otherwise. 
The human tendency is to choose the former option of a certain gain. 
However, if we flip the situation, i.e., a loss of $\$500$ w.p. $1$ versus a loss of $\$5000$ w.p. $0.1$, then humans choose the latter option. 
Observe that the expected gain or loss in each setting is the same ($500$). 
\end{example}

Cumulative prospect theory (CPT) is a risk measure that has been empirically shown to capture human attitude to risk \cite{tversky1992advances, jie2018stochastic}. 
This risk metric uses two \emph{utility functions} $u^+$ and $u^-$, corresponding to gains and losses, and \emph{weight functions} $w^+$ and $w^-$ that reflect the fact that value seen by a human subject is nonlinear in the underlying probabilities \cite{barberis2013thirty}. 

\begin{definition}\label{CPTValueDefn}
The \emph{CPT-value} of a continuous random variable $Y$ is defined as: 
\begin{align}
\rho_{cpt}(Y)&:= \int_0^\infty w^+(\mathbb{P}(u^+(Y) > z))dz \nonumber \\&\qquad \qquad- \int_0^\infty w^-(\mathbb{P}(u^-(Y) > z))dz,  \label{CPTValue}
\end{align}
where utility functions $u^+, u^- : \mathbb{R} \rightarrow \mathbb{R}_{\geq 0}$ are continuous, have bounded first moment such that $u^+(x) = 0$ for all $x \leq 0$, and monotonically non-decreasing otherwise, and $u^-(x) = 0$ for all $x \geq 0$, and monotonically non-increasing otherwise. 
The probability weighting functions $w^+, w^-: [0,1] \rightarrow [0,1]$ are Lipschitz continuous and non-decreasing, and satisfy $w^+(0) = w^-(0) = 0$ and $w^+(1) = w^-(1) = 1$. 
\end{definition}

When $Y$ is a discrete r.v. with finite support, let $p_i$ denote the probability of incurring a gain or loss $y_i$, where $y_1 \leq \dots \leq y_l \leq 0 \leq y_{l+1} \leq \dots y_K$, for $i = 1,2,\dots,K$. 
Define $F_k:= \sum_{i=1}^k p_i$ for $k \leq l$ and $F_k:= \sum_{i=k}^K p_i$ for $k > l$. 

\begin{definition}
The CPT-value of a discrete random variable $Y$ is defined as:
\begin{align}
&\rho_{cpt}(Y)\label{CPTValueDiscrete}\\&:= \bigg(\sum_{i=l+1}^{K-1} u^+(y_i) \big(w^+(F_i) - w^+(F_{i+1}) \big)+u^+(y_K)w^+(p_K) \bigg) \nonumber \\
&\quad- \bigg(u^-(y_1)w^-(p_1)+\sum_{i=2}^{l} u^-(y_i) \big(w^-(F_i) - w^-(F_{i-1}) \big) \bigg) \nonumber
\end{align}
\end{definition}

The function $u^+$ is typically concave on gains, while $-u^-$ is typically convex on losses. 
The distortion of extremely low and extremely high probability events by humans can be represented by a weight function that takes an \emph{inverted S-shape}- i.e., it is concave for small probabilities, and convex for large probabilities. 
When $0 < \eta < 1$, some examples of weighting functions are \cite{tversky1992advances, prelec1998probability}: 
\begin{align*}
&w(\kappa)= \frac{\kappa^\eta}{(\kappa^\eta + (1-\kappa)^\eta)^{\frac{1}{\eta}}}; \quad
w(\kappa)= \exp(-(-\ln \kappa)^\eta).
\end{align*}

The CPT-value generalizes the risk metrics in Example \ref{EgRiskMetrics} for appropriate choices of weighting functions. 
For example, when $w^+, w^-$ are identity functions, and $u^+(x) = x, x \geq 0$, $u^-(x) = -x, x \leq 0$, we obtain $\rho_{cpt}(Y) = \mathbb{E}[Y]$. 

The CPT-value is not a coherent risk metric, since distortion by a nonlinear weighting function will not usually satisfy the Translation invariance and Subadditivity properties. 
However, $\rho_{cpt}$ satisfies the Monotonicity and Positive homogeneity properties \cite{lin2013dynamic}. 

\section{CPT-based Reinforcement Learning}\label{CPTRLResults}

This section introduces a reinforcement learning framework that uses cumulative prospect theory. 
Our objective through this framework is to enable behaviors of an RL agent that will mimic those of a human operator. 
Moreover, behaviors corresponding to operators with different levels of rationality can be achieved by an appropriate choice of weighting function of the CPT-value  \cite{tversky1992advances}.
Specifically, we develop a technique to optimize an accumulated CPT-based cost, and establish conditions under which an iterative procedure describing this technique will converge. 

In order to assess the quality of taking an action $a$ at a state $s$, we introduce the notion of the \emph{CPT-value of state-action pair at time $t$ and following policy $\pi$} subsequently. 
We denote this by $Q^\pi_{cpt}(s,a)$ and will refer to it as \emph{CPT-Q}. 
\emph{CPT-Q} is defined in the following manner: 
\begin{align}
&Q^\pi_{cpt}(s_t,a_t):= \rho_{cpt}(c(s_t, a_t)\label{CPT-Q-Iterative}\\ &+ \gamma \sum_{s_{t+1}} \mathbb{P}(s_{t+1}|s_t,a_t) \sum_{a_{t+1}} \pi(a_{t+1}|s_{t+1})Q^\pi_{cpt}(s_{t+1},a_{t+1})). \nonumber
\end{align}

$Q^\pi_{cpt}(s,a)$ will be bounded when $|c(s,a)| < \infty$ and $\gamma \in (0,1)$. 
In reinforcement learning, transition probabilities and costs are typically not known apriori. 
In the absence of a model, the agent will have to estimate $Q^\pi_{cpt}(s,a)$ and learn `good' policies by exploring its environment. 
Since $Q^\pi_{cpt}(s,a)$ is evaluated for each action in a state, this quantity can be estimated without knowledge of the transition probabilities. 
This is in contrast to \cite{lin2018probabilistically}, where a model of the system was assumed to be available, and costs were known.

The \emph{CPT-value of a state $s$ when following policy $\pi$} is defined as $V^\pi_{cpt}(s_t) := \sum_{a_t} \pi(a_t|s_t)Q^\pi_{cpt}(s_t,a_t)$. 
We will refer to $V^\pi_{cpt}(s)$ as \emph{CPT-V}. 
We observe that \emph{CPT-V} satisfies: 
\begin{align}
V^\pi_{cpt}(s_t)&= \rho_{cpt}(c(s_t, a^\pi_t)\label{CPT-V-Iterative}\\&\qquad \qquad+\gamma \sum_{s_{t+1}} \mathbb{P}(s_{t+1}|s_t,a^\pi_t) V^\pi_{cpt}(s_{t+1})). \nonumber
\end{align}
Denote the minimum \emph{CPT-V }at a state $s$ by $V^*_{cpt}(s)$. 
Then, $V^*_{cpt}(s) = \inf_\pi V^\pi_{cpt}(s)$. 

\begin{remark}
To motivate the construction of this framework, let the random variable $C(s_0) = \sum_{i=0}^\infty \gamma^i c(s_i, a^\pi_i)$ denote the infinite horizon cumulative discounted cost starting from state $s_0$. 
The objective in a typical RL problem is 
to determine a policy $\pi$ to minimize the expected cost, denoted $\mathbb{E}_\pi[C(s_0)]$. 
The linearity of the expectation operator allows us to write $\mathbb{E}_\pi[C(s_0)] = \mathbb{E} [c(s_0, a_0) + \gamma \mathbb{E} [c(s_1, a_1) + \dots |s_1] |s_0]$. 
In this work, we are interested in minimizing the sum of CPT-based costs over the horizon of interest. 
This will correspond to replacing the conditional expectation at each time-step with $\rho_{cpt}(\cdot)$.
\end{remark}

In order to show convergence of CPT-Q-learning in Equation (\ref{CPT-Q-Iterative}), we introduce the \emph{CPT-Q-iteration} operator as: 
\begin{align}
(\mathcal{T}_\pi Q^\pi_{cpt})(s,a)&:=\rho_{cpt}(c(s, a) \label{CPT-Q-Operator}\\ &+ \gamma \sum_{s'} \mathbb{P}(s'|s,a) \sum_{a'} \pi(a'|s')Q^\pi_{cpt}(s',a')). \nonumber
\end{align}
We next show that $(\mathcal{T}_\pi Q_{cpt})$ is monotone and is a contraction. 
We first define a notion of policy improvement, and present a sufficient condition for a policy to be improved. 

\begin{definition}
A policy $\pi'$ is said to be \emph{improved} compared to policy $\pi$ if and only if for all $s \in S$, $V^{\pi'}_{cpt}(s) \leq V^\pi_{cpt}(s)$. 
\end{definition}

\begin{proposition}
Consider a policy $\pi'$ such that $\pi'$ is different from $\pi$ at step $t$, and is identical (in distribution) to $\pi$ for all subsequent steps. 
If $\sum_{a_t} \pi'(a_t|s_t)Q^\pi_{cpt}(s_t,a_t) \leq V^{\pi}_{cpt}(s_t)$ for all $s_t \in S$, then $\pi'$ is improved compared to $\pi$. 
\end{proposition}

\begin{proof}
From Equation (\ref{CPT-Q-Iterative}),
\begin{align}
&\sum_{a_t} \pi'(a_t|s_t)Q^\pi_{cpt}(s_t,a_t)\nonumber \\& \quad= \sum_{a_t} \pi'(a_t|s_t)[\rho_{cpt}(c(s_t, a_t) \nonumber\\ &+ \gamma \sum_{s_{t+1}} \mathbb{P}(s_{t+1}|s_t,a_t) \sum_{a_{t+1}} \pi(a_{t+1}|s_{t+1})Q^\pi_{cpt}(s_{t+1},a_{t+1}))]. \nonumber
\end{align}
Since $\pi'$ is identical to $\pi$ for all steps beyond $t$, the above expression is equivalent to:
\begin{align}
&\sum_{a_t} \pi'(a_t|s_t)[\rho_{cpt}(c(s_t, a_t) \nonumber + \gamma \sum_{s_{t+1}} \mathbb{P}(s_{t+1}|s_t,a_t)V^{\pi'}_{cpt}(s_{t+1}))].\nonumber
\end{align}
This quantity is equal to $V^{\pi'}_{cpt}(s_t)$. 
Therefore, we have $\sum_{a_t} \pi'(a_t|s_t)Q^\pi_{cpt}(s_t,a_t)= V^{\pi'}_{cpt}(s_t) \leq V^{\pi}_{cpt}(s_t)$ for all $s_t \in S$. 
Thus, taking an action according to policy $\pi'$ at time $t$ and following the original policy $\pi$ at subsequent time-steps ensures that the value of state $s_t$ is lower. This indicates that $\pi'$ is improved compared to $\pi$, completing the proof.
\end{proof}

\begin{proposition}\label{PropMonotone}
Let policies $\pi$ and $\pi'$ be such that $Q^{\pi'}_{cpt}(s,a) \leq Q^\pi_{cpt}(s,a)$ for all $(s,a) \in S \times A$, and $\pi'$ is improved compared to $\pi$. 
Let the functions $w^+, w^-, u^+, u^-$ be according to Definition \ref{CPTValueDefn}. 
Then, $(\mathcal{T}_{\pi'} Q^{\pi'}_{cpt}) \leq (\mathcal{T}_\pi Q^\pi_{cpt})$. 
\end{proposition}

\begin{proof}
Since the utility function $u^+$ is monotonically non-decreasing, $Q^{\pi'}_{cpt}(s,a) \leq Q^\pi_{cpt}(s,a)$, and $\pi'$ is improved compared to $\pi$, we have:
\begin{align*}
&u^+(c(s, a) + \gamma \sum_{s'} \mathbb{P}(s'|s,a) \sum_{a'} \pi'(a'|s')Q^{\pi'}_{cpt}(s',a'))\\\leq &u^+(c(s, a) + \gamma \sum_{s'} \mathbb{P}(s'|s,a) \sum_{a'} \pi(a'|s')Q^{\pi}_{cpt}(s',a')).
\end{align*}
We represent the above inequality as $u^+_{\pi'} \leq u^+_{\pi}$. 
Since the probability weighting function $w^+$ is also monotonically non-decreasing, we have:
\begin{align*}
&\int_0^\infty w^+(\mathbb{P}(u^+_{\pi'} > z))dz \leq \int_0^\infty w^+(\mathbb{P}(u^+_\pi > z))dz.
\end{align*}
A similar argument will hold for the functions $u^-$ and $w^-$, and therefore, $(\mathcal{T}_{\pi'} Q^{\pi'}_{cpt}) \leq (\mathcal{T}_\pi Q^\pi_{cpt})$. This shows that operator $(\mathcal{T}_\pi Q_{cpt})$ is monotone.  
\end{proof}

\begin{proposition}\label{PropContract}
Let the functions $w^+, w^-, u^+, u^-$ be according to Definition \ref{CPTValueDefn}. 
Assume that the utility functions $u^+, u^-$ are invertible and differentiable, and that the derivatives are monotonically non-increasing. 
Then, the operator $(\mathcal{T}_\pi Q_{cpt})$ is a contraction.
\end{proposition}

\begin{proof}
First, we define a norm on the CPT-Q values as $||Q_{cpt}^1 - Q_{cpt}^2||:=\max_{s,a}|Q_{cpt}^1(s,a) - Q_{cpt}^2(s,a)|$, and suppose $\epsilon : =||Q_{cpt}^1 - Q_{cpt}^2||$. Then, 
\begin{align*}
&(\mathcal{T}_\pi Q^1_{cpt})(s,a):=\rho_{cpt}(c(s, a)  \\ &\qquad \qquad \qquad \qquad + \gamma \sum_{s'} \mathbb{P}(s'|s,a) \sum_{a'} \pi(a'|s')Q^1_{cpt}(s',a'))\\
&= \rho_{cpt}(c(s, a) + \gamma \sum_{s'} \mathbb{P}(s'|s,a) \sum_{a'} \pi(a'|s') \\&\qquad \qquad  \qquad \qquad (Q^2_{cpt}(s',a')+Q^1_{cpt}(s',a')-Q^2_{cpt}(s',a'))\\
&\leq \rho_{cpt}(\gamma \epsilon + c(s, a)  + \gamma \sum_{s'} \mathbb{P}(s'|s,a) \sum_{a'} \pi(a'|s')Q^2_{cpt}(s',a')).
\end{align*}
The remainder of the proof follows by considering each of the integrals that make up $\rho_{cpt}(\cdot)$ separately, and using the assumptions on $w^+, w^-, u^+, u^-$ to obtain $(\mathcal{T}_\pi Q^1_{cpt})(s,a) \leq (\mathcal{T}_\pi Q^2_{cpt})(s,a) + \gamma \epsilon$. 
We refer to Theorem 6 in \cite{lin2018probabilistically} for details on this procedure\footnote{We observe that if $\rho_{cpt}$ had satisfied the Translation Invariance property, then the result would have followed directly, like in \cite{shen2014risk}.}.
A consequence of this analysis is that we obtain $\||\mathcal{T}_\pi Q^1_{cpt} - \mathcal{T}_\pi Q^2_{cpt}|| \leq \gamma ||Q_{cpt}^1 - Q_{cpt}^2||$, which shows that $(\mathcal{T}_\pi Q_{cpt})$ is a contraction.  
\end{proof}

In order to allow the agent to learn policies through repeated interactions with this environment, consider the following iterative procedure:
\begin{align}
&Q^{k+1}_{cpt}(s_t,a_t) = (1-\alpha_k(s_t,a_t))Q^k_{cpt}(s_t,a_t) \label{Q-Iteration}\\&\qquad \qquad+ \alpha_k(s_t,a_t)[\rho_{cpt}(c(s_t, a_t) \nonumber \\&\qquad \qquad \quad+ \gamma \sum_a \pi(a|s_{t+1}) Q^k_{cpt}(s_{t+1},a))].\nonumber
\end{align}
In Equation (\ref{Q-Iteration}), $\alpha(s_t,a_t) \in (0,1)$ is a \emph{learning rate} which determines state-action pairs whose $Q-$values are updated at iteration $k$. 
The learning rate for a state-action pair is typically inversely proportional to the number of times the pair is visited during the exploration phase. 

The next result presents a guarantee that the sequence of CPT-Q-values in Equation (\ref{Q-Iteration}) will converge to a unique solution under the assumption that state-action pairs of a finite MDP are visited infinitely often. 

\begin{proposition}\label{PropConvergence}
For an MDP with finite state and action spaces, assume that the costs $c(s,a)$ are bounded for all $(s,a)$, and learning rates satisfy for all $(s,a)$, $\sum_k \alpha_k (s,a) = \infty$, $\sum_k \alpha_k^2(s,a) < \infty$, and that the operator $(\mathcal{T}_\pi Q_{cpt})$ is a contraction. 
Then, the CPT-Q-iteration in Equation (\ref{Q-Iteration}) will converge to a unique solution $Q^*_{cpt}(s,a)$ for each $(s,a)$ with probability one. 
\end{proposition}

\begin{proof}
From Proposition \ref{PropContract}, we know that repeated application of $(\mathcal{T}_\pi Q_{cpt})$ results in convergence to a fixed point $Q^*_{cpt}$ which satisfies $(\mathcal{T}_\pi Q^*_{cpt}) =  Q^*_{cpt}$ for all $(s,a) \in S \times A$. 
Defining $\Delta_k := Q^k_{cpt} - Q^*_{cpt}$, Eqn (\ref{Q-Iteration}) can be written as: 
\begin{align*}
\Delta_{k+1}&=(1-\alpha_k(s_t,a_t))\Delta_k + \alpha_k(s_t,a_t)[\mathcal{T}_\pi Q^k_{cpt} - \mathcal{T}_\pi Q^*_{cpt}].
\end{align*}
The above equation is in the form of a stochastic iterative process on $\Delta_k$. 
Since $(\mathcal{T}_\pi Q_{cpt})$ is a contraction and the costs are bounded, the sequence $\{\Delta_k\}$ will converge to zero with probability one \cite{jaakkola1994convergence}. 
\end{proof}

\section{Algorithms for CPT-based RL}\label{CPTRLAlgos}
\begin{algorithm}[!h]
	\small
	\caption{CPT-Estimation}
	\label{algo:CPT-Estimation}
	\begin{algorithmic}[1]
		\REQUIRE{State $s$, action $a$, current policy $\pi$,  max. samples $N_{max}$}
		\STATE{\textbf{Initialize} $n = 1$; $X_{0}:=\infty$; $s_* \leftarrow s$}
		\REPEAT
		\STATE{Take action $a$, observe $c(s,a)$ and next state $s'$}
		\STATE{$X_n:=c(s, a) + \gamma \sum_b \pi(b|s') Q_{cpt}(s', b)$}
		\IF{$X_n < X_{0}$}
		\STATE{$s_* \leftarrow s'$}
		\STATE{$X_0 \leftarrow X_n$}
		\ENDIF
		\STATE{$n \leftarrow n+1$}
		\UNTIL{$n>N_{max}$}
		\STATE{Arrange samples $\{X_i\}$ in ascending order: $X_{[1]} \leq X_{[2]} \leq \dots$}
		\STATE{Let:
		\begin{align*}
		\rho_{cpt}^+:&=\sum_{i=1}^{N_{max}} u^+(X_{[i]})(w^+(\frac{N_{max}+i-1}{N_{max}}) - w^+(\frac{N_{max}-i}{N_{max}}))\\
		\rho_{cpt}^-:&=\sum_{i=1}^{N_{max}} u^-(X_{[i]})(w^-(\frac{i}{N_{max}}) - w^-(\frac{i-1}{N_{max}}))
		\end{align*}		
		}
		\STATE{$\rho_{cpt}(c(s,a)+\gamma \sum_b \pi(b|\cdot) Q_{cpt}(\cdot, b)):= \rho_{cpt}^+ - \rho_{cpt}^-$}
		\RETURN{$\rho_{cpt}(\cdot), s_*$}
	\end{algorithmic}
\end{algorithm}

\begin{algorithm}[!h]
	\small
	\caption{CPT-SARSA}
	\label{algo:CPT-SARSA}
	\begin{algorithmic}[1]
		\REQUIRE{Learning rate $\alpha$; max. episodes $T_{max}$; discount $\gamma$}
		\STATE{\textbf{Initialize} $Q_{cpt}(s,a)$, $T = 1$}
		\REPEAT
		\STATE{Initialize $s \in S$}
		\REPEAT
		\STATE{Choose $a$ according to policy $\pi$}
		\STATE{Obtain $\rho_{cpt}(\cdot), s_*$ from Algorithm \ref{algo:CPT-Estimation}}
		\STATE{$\delta:=\rho_{cpt}(\cdot)-Q_{cpt}(s,a)$}
		\STATE{$Q_{cpt}(s,a) \leftarrow Q_{cpt}(s,a) + \alpha \delta$}
		\STATE{$s \leftarrow s_*$}
		\UNTIL{$s$ is a terminal state}
		\STATE{$T \leftarrow T+1$}
		\UNTIL{$T>T_{max}$}
	\end{algorithmic}
\end{algorithm}

In this section, we will present two algorithms using temporal difference (TD) techniques for CPT-based reinforcement learning. 
TD techniques seek to learn value functions using episodes of experience.
An experience episode comprises a sequence of states, actions, and costs when following a policy $\pi$. 
The predicted values at any time-step is updated in a way to bring it closer to the prediction of the same quantity at the next time-step. 

Formally, given that the RL agent in a state $s$ took action $a$ and transitioned to state $s'$, the \emph{TD-update} of $V_{cpt}(s)$ is given by $V_{cpt}(s) \leftarrow V_{cpt}(s)+\alpha( \rho_{cpt}(c(s,a)+\gamma V_{cpt}(s'))-V_{cpt}(s))$. 
This can be rewritten as $V_{cpt}(s) \leftarrow V_{cpt}(s)+\alpha \delta$, where $\delta$ is called the \emph{TD-error}. 
A positive value of $\delta$ indicates that the action taken in state $s$ resulted in an improved CPT-V. 
The TD-update in this case determines estimates of CPT-V. 
A similar update rule can be defined for CPT-Q. 
From Equations (\ref{CPTValue}) and (\ref{CPTValueDefn}), we observe that  $\rho_{cpt}$ is defined in terms of a weighting function applied to a cumulative probability distribution. 
In order to use TD-methods in a prospect-theoretic framework, we first use a technique proposed in \cite{jie2018stochastic} to estimate the CPT-value $\rho_{cpt}$. 

\subsection{Calculating $\rho_{cpt}$ from samples}

Algorithm \ref{algo:CPT-Estimation} is a procedure to obtain multiple samples of the random variable $c(s,a)+\gamma V_{cpt}(s')$. 
These samples are then used to estimate $\rho_{cpt}(c(s,a)+\gamma V_{cpt}(s'))$. 
This way to estimate the CPT-value of a random variable was proposed in \cite{jie2018stochastic}, and was shown to be asymptotically consistent. 

\subsection{CPT-SARSA}

Algorithm \ref{algo:CPT-SARSA} updates an estimate of CPT-Q for each state-action pair $(s,a)$ through a temporal difference. 
`SARSA' is named for the components used in the update: the state, action, and reward at time $t$ (S-A-R-), and the state and action at time $t+1$ (-S-A) \cite{van2009theoretical}. 
For an action $a$ taken in state $s$ resulting in a transition to a state $s'$, CPT-SARSA exploits the randomized nature of the policy to compute a weighted sum of possible actions in state $s'$ according to the current policy $\pi$ (Line 4). 

For a state-action pair $(s,a)$, the learning rate $\alpha (s,a)$ is set to $1/N(s,a)$, where $N(s,a)$ is a count of the number of times $(s,a)$ is visited. 
$N(s,a)$ is incremented by $1$ each time $(s,a)$ is visited, and therefore, $\alpha (s,a)$ will satisfy the conditions of Proposition \ref{PropConvergence}. 

\subsection{CPT-Actor-Critic}

Actor-critic methods separately update state-action values and a parameter associated to the policies. 
Algorithm \ref{algo:CPT-Actor-Critic} indicates how these updates are carried out. 
The \emph{actor} is a policy $\pi$ while the \emph{critic} is an estimate of the value function $Q_{cpt}^\pi$. 
This setting can be interpreted in the following way: if the updated critic for action $a$ at state $s$, obtained by computing the TD-error, is a quantity that is higher than the value obtained by taking a reference action $a_{ref}$ at state $s$, then $a$ is a `good' action. 
As a consequence, the tendency to choose this action in the future can be increased. 
Under an assumption that learning rates for the actor and critic each satisfy the conditions of Proposition \ref{PropConvergence}, and that the actor is updated at a much slower rate than the critic, this method is known to converge \cite{borkar2000ode}. 

\begin{algorithm}[!h]
	\small
	\caption{CPT-Actor-Critic}
	\label{algo:CPT-Actor-Critic}
	\begin{algorithmic}[1]
		\REQUIRE{Learning rates $\alpha_1, \alpha_2$; max. episodes $T_{max}$; discount $\gamma$}
		\STATE{\textbf{Initialize} $Q_{cpt}(s,a)$, $p(s,a)$, $T = 1$}
		\REPEAT
		\STATE{Initialize $s \in S$}
		\REPEAT
		\STATE{Choose $a$ according to policy $\pi$}
		\STATE{Obtain $\rho_{cpt}(\cdot), s_*$ from Algorithm \ref{algo:CPT-Estimation}}
		\STATE{$\delta:=\rho_{cpt}(\cdot)-Q_{cpt}(s,a)$}
		\STATE{$Q_{cpt}(s,a) \leftarrow Q_{cpt}(s,a) + \alpha_1 \delta$}
		\STATE{$p(s,a) \leftarrow p(s,a) + \alpha_2 (Q_{cpt}(s,a) - Q_{cpt}(s,a_{ref}))$}
		\STATE{Generate new policy $\pi'$ using updated $p(s,a)$} 
		\STATE{$s \leftarrow s_*$, $\pi \leftarrow \pi'$}
		\UNTIL{$s$ is a terminal state}
		\STATE{$T \leftarrow T+1$}
		\UNTIL{$T>T_{max}$}
	\end{algorithmic}
\end{algorithm}

One example in which the policy parameters $p(s,a)$ can be used to determine a policy is the Gibbs softmax method \cite{sutton2018reinforcement}, defined as $\pi(a|s):= \frac{exp(p(s,a))}{\sum_{b \in A} exp(p(s,b))}$. 

\begin{figure}[t!]
	\centering
	\begin{subfigure}{.23\textwidth}
		\includegraphics[width=\textwidth]{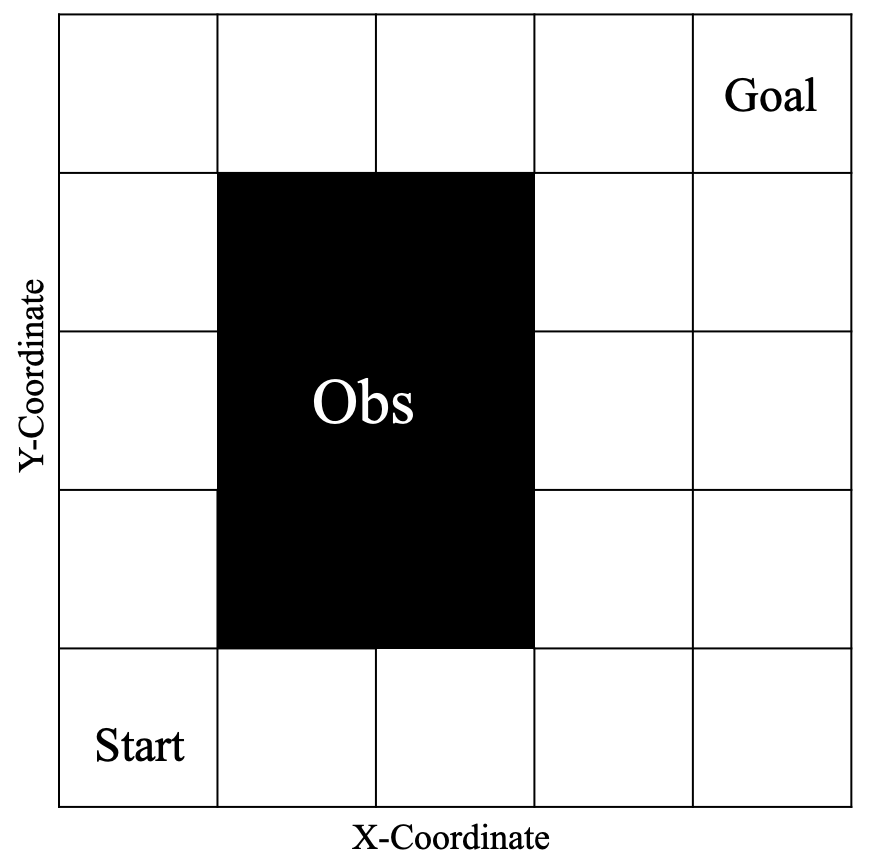}
		\subcaption {}
		\label{fig:grid world 1}
	\end{subfigure}\hfill
	\begin{subfigure}{.24\textwidth}
		\includegraphics[width=\textwidth]{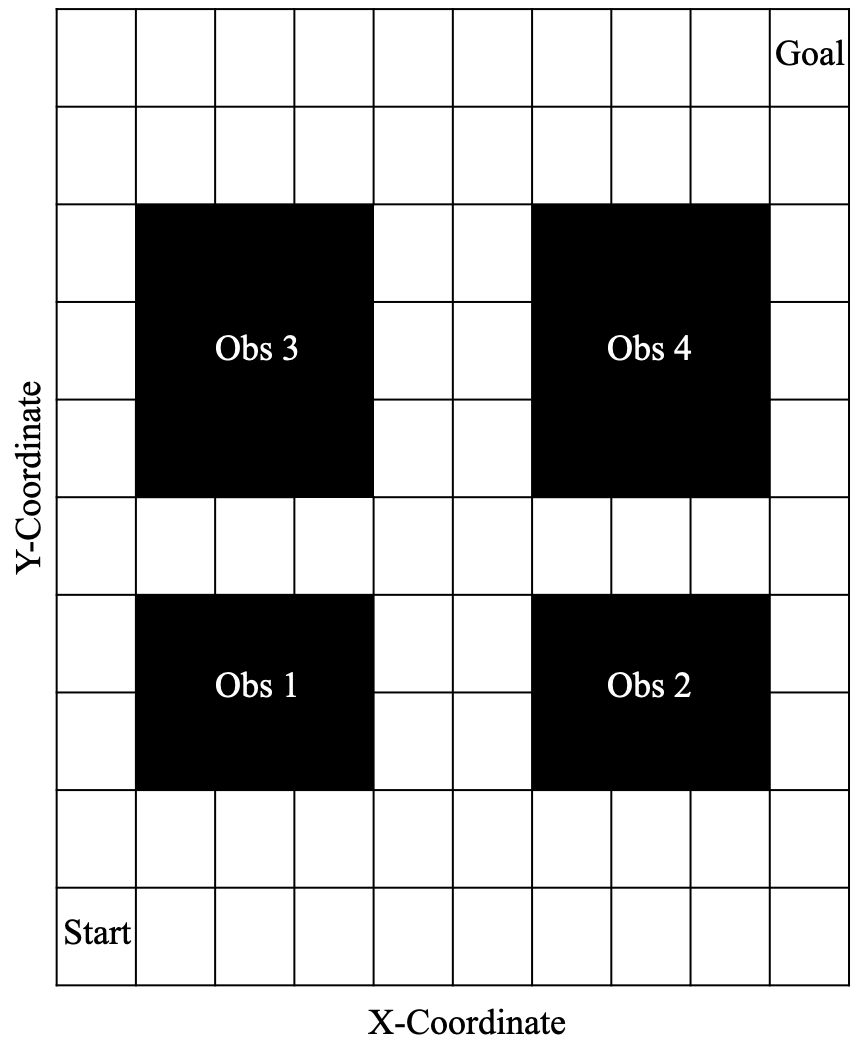}
		\subcaption {}
		\label{fig:grid world 2}
	\end{subfigure}\hfill
	\caption{The environments on which the CPT-based RL algorithms are evaluated. The agent has to learn to reach the `Goal' from the `Start' while avoiding the obstacles. Fig. \ref{fig:grid world 1} shows an environment with a single obstacle. Fig. \ref{fig:grid world 2} is a larger environment with four obstacles. The agent incurs a different cost for encountering each of these obstacles.}
\end{figure}
\begin{figure*}[t!]
	\centering
	\begin{subfigure}{.43\textwidth}
		\includegraphics[width=\textwidth]{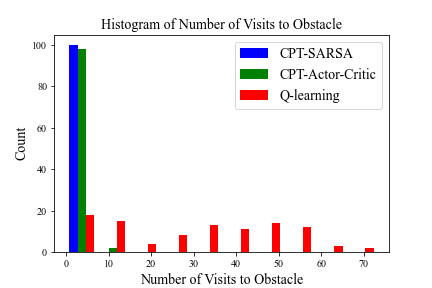}
		\subcaption {}
		\label{fig:collision 1}
	\end{subfigure}\hfill
	\begin{subfigure}{.43\textwidth}
		\includegraphics[width=\textwidth]{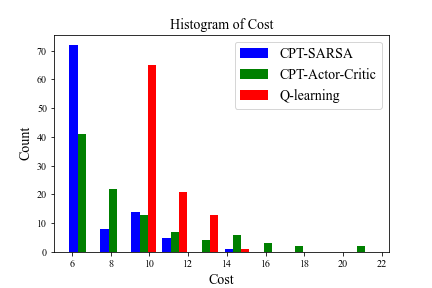}
		\subcaption{}
		\label{fig:cost 1}
	\end{subfigure}\hfill
	\caption{Comparison of agent adopting policies learned using CPT-SARSA, CPT-Actor-Critic, and Q-learning for environment in Fig. \ref{fig:grid world 1}. 
	For a policy learned using each method, we generate $100$ sample paths. 
	Fig. \ref{fig:collision 1} compares the number of times that the obstacle region is reached on each sample path.  
	In almost all $100$ cases, the agent using CPT-SARSA (blue) or CPT-Actor-Critic (green) is able to avoid the obstacle, while this is not the case when it uses Q-learning (red). 
	Fig. \ref{fig:cost 1} compares the total cost incurred by the agent in reaching the target state from the start state. 
	We observe that in some cases, the agent using a CPT-based method incurs a higher cost than when using Q-learning. 
	This can be interpreted in terms of the agent taking a possibly longer route to reach the `Goal' state in order to avoid the obstacle.
	}\label{FigHist1}
\end{figure*}
\begin{figure*}[t!]
	\centering
	\begin{subfigure}{.43\textwidth}
		\includegraphics[width=\textwidth]{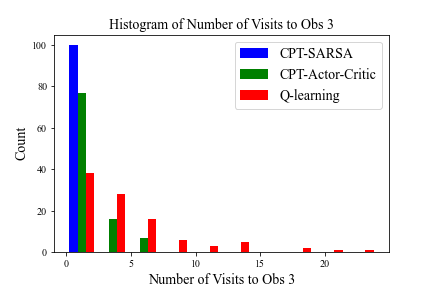}
		\subcaption{}
		\label{fig:collision 2_3}
	\end{subfigure}\hfill
	\begin{subfigure}{.43\textwidth}
		\includegraphics[width=\textwidth]{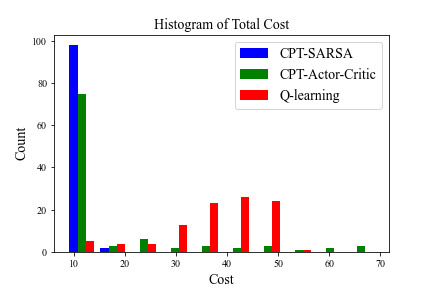}
		\subcaption {}
		\label{fig:cost 2}
	\end{subfigure}\hfill
	\caption{Comparison of agent adopting policies learned using CPT-SARSA, CPT-Actor-Critic, and Q-learning for environment in Fig. \ref{fig:grid world 2}. 
	For a policy learned using each method, we generate $100$ sample paths. 
	Fig. \ref{fig:collision 2_3} compares the number of times that the obstacle $Obs$ $3$ is reached on each sample path.  
	In almost all $100$ cases, the agent using CPT-SARSA (blue) or CPT-Actor-Critic (green) is able to avoid the obstacle, while this is not the case when it uses Q-learning (red). 
	Fig. \ref{fig:cost 2} compares the total cost incurred by the agent in reaching the target state from the start state.}\label{FigHist2}
\end{figure*}

\section{Experimental Evaluation}\label{Simulation}

This section presents an evaluation of the CPT-based RL algorithms developed in the previous section in two environments. 
The environments are represented as a grid and the agent has the ability to move in the four cardinal directions ($\{left,right,up,down\}$). 
There are obstacles that the agent will have to avoid in order to reach a target state. 
In each case, a model of the environment is not available to the agent, and the agent will have to learn behaviors through minimizing an accumulated CPT-based cost, where the cost signal is provided by the environment. 
We compare behaviors learned when the agent optimizes a CPT-based cost with a baseline when the agent optimizes an expected cost.

\subsection{Environments}

We evaluate our methods on two environments shown in Figures \ref{fig:grid world 1} and \ref{fig:grid world 2}. 
In each case, the agent has to learn to reach a target state while avoiding the obstacles. 
We assume that the agent starts from the state `Start' at the bottom left corner, and the target state `Goal' is at the top right corner. 
At each state, the agent can take one of four possible actions, $\{left,right,up,down\}$. 
A \emph{neighboring state} is defined as any state that shares a boundary with the current state of the agent, and we denote the number of neighboring states at the current state by $N_{ns}$. 
For each action that the agent takes at a state, the agent can transit with probability $0.9$ to the intended next state, and with probability $0.1/(N_{ns}-1)$ to some neighboring state. 
However, this transition probability is not known to the agent. 

We compare the behaviors learned by the agent using the CPT-SARSA and CPT-Actor-Critic algorithms with a baseline where the agent uses Q-learning \cite{sutton2018reinforcement} to minimize a total discounted cost. 
The discount factor $\gamma$ is set to $0.9$, and the utility and weighting functions for the CPT-based methods are chosen as: 
\begin{align*}
&u^+(x) = |x|^{0.88}; \quad u^-(x)=|x|^{0.88};\\
&\omega^+(\kappa) = \frac{\kappa^{0.61}}{(\kappa^{0.61}+(1-\kappa)^{0.61})^{\frac{1}{0.61}}}; \\ &\omega^-(\kappa) = \frac{\kappa^{0.69}}{(\kappa^{0.69}+(1-\kappa)^{0.69})^{\frac{1}{0.69}}}.
\end{align*}

In both environments, the cost incurred by the agent is $1$ at a state that is not an obstacle. 
For the environment in Figure \ref{fig:grid world 1}, the cost is $5$ at a state occupied by the obstacle. 
For the environment in Figure \ref{fig:grid world 2}, the cost incurred by the agent at an obstacle $Obs$ $z$ is given by $10*z$, where $z = \{1,2,3,4\}$. 

\subsection{Results}

Figure \ref{FigHist1} compares behaviors learned by the agent when adopting policies learned using the CPT-SARSA, CPT-Actor-Critic, and Q-learning methods for the environment in Fig. \ref{fig:grid world 1}. 
We generate $100$ sample paths for a policy learned using each method, and compare the number of times the agent visits the obstacle, and the total cost incurred by the agent to reach the target state. 
We observe in Fig. \ref{fig:collision 1} that the agent using either CPT-based method is able to avoid the obstacle in almost all cases. 
This might however come at a possibility of incurring a higher cost in some cases, as seen in Fig. \ref{fig:cost 1}.  
The behavior exhibited by the agent in this case is better aligned with the behavior of a human user who when placed in the same environment, might prefer to take a longer route from `Start' to `Goal' in order to avoid the obstacle. 

The behavior of the agent in the environment in Fig. \ref{fig:grid world 2} is shown in Fig. \ref{FigHist2}.  
Histograms for visits to $Obs$ $1$ and $Obs$ $2$ are similar to that shown in Fig. \ref{fig:collision 2_3}. 
The agent using a CPT-based algorithm is able to avoid obstacles and incur a lower cost in most cases than an agent using Q-learning in this environment as well. 
An agent using CPT-Actor-Critic incurs a higher cost and visits an obstacle more number of times than when using CPT-SARSA. 
This could be because CPT-Actor-Critic requires adjusting values of two learning rates. 
The choice of reference action $a_{ref}$ in Line 9 of Algorithm \ref{algo:CPT-Actor-Critic} may also play a role.
\begin{table}[h!]
	\centering
	\begin{tabular}{||c |c| c| c| c||} 
		\hline
		& Obs 1 & Obs 2 & Obs 3 & Obs 4  \\ [0.5ex] 
		\hline
		CPT-SARSA & 0.01 & 0.04 & 0.03  & 0 \\ 
		\hline
		CPT-Actor-Critic & 1.87 & 0.13 & 1.42 & 0  \\
		\hline
		Q-learning  & 29.82 & 2.81 & 4.55 & 0\\
		\hline
	\end{tabular}
	\caption{Number of visits to obstacle regions, averaged over 100 sample paths, for environment in Fig. \ref{fig:grid world 2}.}
	\label{table:compare}
\end{table}

For the environment in Fig. \ref{fig:grid world 2}, the number of visits to obstacles, averaged over $100$ sample paths for policies adopted following each method are presented in Table \ref{table:compare}. 
%
All three methods allow the agent to learn policies to avoid the most costly obstacle $Obs$ $4$. However, the number of times obstacles with lower costs are encountered is much smaller for CPT-SARSA and CPT-Actor-Critic than Q-learning. 
The largest difference in the number of visits is observed for the obstacle with lowest cost, $Obs$ $1$. 

\section{Conclusion}\label{Conclusion}

This paper presented a way to enable a reinforcement learning (RL) agent learn behaviors to closely mimic that of a human user. 
The ability of autonomous agents to align their behaviors with human users is becoming increasingly important in situations where actions of human users and an autonomous system influence each other in a shared environment. 
We used cumulative prospect theory (CPT) to model the tendency of humans to view gains and losses differently. 
When the agent had to learn behaviors in an unknown environment, we used a CPT-based cost to specify the objective that the agent had to minimize. 
We developed two RL algorithms to enable the agent to learn policies to optimize the CPT-value of a state-action pair, and evaluated these algorithms in two environments. 
We observed that the behaviors of the agent when following policies learned using the CPT-based methods were better aligned with those of a human user who might be placed in the same environment, and is significantly improved over a Q-learning baseline. 

Our analysis and experiments in this paper considered discrete state and action spaces. 
We will seek to extend our work to settings with continuous states and actions. 
A second research direction is to study the case when value functions and policies will be parameterized by neural networks. 

\bibliographystyle{IEEEtran}
\bibliography{CDC21References}

\begin{thebibliography}{10}
\providecommand{\url}[1]{#1}
\csname url@samestyle\endcsname
\providecommand{\newblock}{\relax}
\providecommand{\bibinfo}[2]{#2}
\providecommand{\BIBentrySTDinterwordspacing}{\spaceskip=0pt\relax}
\providecommand{\BIBentryALTinterwordstretchfactor}{4}
\providecommand{\BIBentryALTinterwordspacing}{\spaceskip=\fontdimen2\font plus
\BIBentryALTinterwordstretchfactor\fontdimen3\font minus
  \fontdimen4\font\relax}
\providecommand{\BIBforeignlanguage}[2]{{%
\expandafter\ifx\csname l@#1\endcsname\relax
\typeout{** WARNING: IEEEtran.bst: No hyphenation pattern has been}%
\typeout{** loaded for the language `#1'. Using the pattern for}%
\typeout{** the default language instead.}%
\else
\language=\csname l@#1\endcsname
\fi
#2}}
\providecommand{\BIBdecl}{\relax}
\BIBdecl

\bibitem{sutton2018reinforcement}
R.~S. Sutton and A.~G. Barto, \emph{Reinforcement Learning: {A}n
  Introduction}.\hskip 1em plus 0.5em minus 0.4em\relax MIT Press, 2018.

\bibitem{bertsekas2017dynamic}
D.~P. Bertsekas, \emph{Dynamic {P}rogramming and {O}ptimal {C}ontrol, 4th
  Ed.}\hskip 1em plus 0.5em minus 0.4em\relax Athena Scientific, 2017, vol.~1.

\bibitem{puterman2014markov}
M.~L. Puterman, \emph{Markov decision processes: {D}iscrete stochastic dynamic
  programming}.\hskip 1em plus 0.5em minus 0.4em\relax John Wiley \& Sons,
  2014.

\bibitem{hafner2011reinforcement}
R.~Hafner and M.~Riedmiller, ``Reinforcement learning in feedback control,''
  \emph{Machine Learning}, vol.~84, pp. 137--169, 2011.

\bibitem{mnih2015human}
V.~Mnih \emph{et~al.}, ``Human-level control through deep reinforcement
  learning,'' \emph{Nature}, vol. 518, no. 7540, 2015.

\bibitem{silver2016mastering}
D.~Silver \emph{et~al.}, ``Mastering the game of {G}o with deep neural networks
  and tree search,'' \emph{Nature}, vol. 529, no. 7587, 2016.

\bibitem{zhang2019deep}
C.~Zhang, P.~Patras, and H.~Haddadi, ``Deep learning in mobile and wireless
  networking: {A} survey,'' \emph{IEEE Communications Surveys \& Tutorials},
  vol.~21, no.~3, pp. 2224--2287, 2019.

\bibitem{sadigh2016planning}
D.~Sadigh, S.~Sastry, S.~A. Seshia, and A.~D. Dragan, ``Planning for autonomous
  cars that leverage effects on human actions.'' in \emph{Robotics: Science and
  Systems}, 2016.

\bibitem{yan2018data}
Z.~Yan and Y.~Xu, ``Data-driven load frequency control for stochastic power
  systems: {A} deep reinforcement learning method with continuous action
  search,'' \emph{IEEE Transactions on Power Systems, 34(2)}, 2018.

\bibitem{you2019advanced}
C.~You, J.~Lu, D.~Filev, and P.~Tsiotras, ``Advanced planning for autonomous
  vehicles using reinforcement learning and deep inverse {RL},'' \emph{Robotics
  and Autonomous Systems}, vol. 114, pp. 1--18, 2019.

\bibitem{gollier2001economics}
C.~Gollier, \emph{The economics of risk and time}.\hskip 1em plus 0.5em minus
  0.4em\relax MIT Press, 2004.

\bibitem{gilboa2009theory}
I.~Gilboa, \emph{Theory of decision under uncertainty}.\hskip 1em plus 0.5em
  minus 0.4em\relax Cambridge University Press, 2009.

\bibitem{shen2013risk}
Y.~Shen, W.~Stannat, and K.~Obermayer, ``Risk-sensitive {M}arkov control
  processes,'' \emph{SIAM Journal on Control and Optimization}, vol.~51, no.~5,
  pp. 3652--3672, 2013.

\bibitem{shen2014risk}
Y.~Shen, M.~J. Tobia, T.~Sommer, and K.~Obermayer, ``Risk-sensitive
  reinforcement learning,'' \emph{Neural computation}, vol.~26, no.~7, pp.
  1298--1328, 2014.

\bibitem{kahneman1979prospect}
D.~Kahneman and A.~Tversky, ``Prospect theory: {A}n analysis of decision under
  risk,'' \emph{Econometrica}, vol.~47, no.~2, pp. 263--292, 1979.

\bibitem{tversky1992advances}
A.~Tversky and D.~Kahneman, ``Advances in prospect theory: {C}umulative
  representation of uncertainty,'' \emph{Journal of Risk and uncertainty},
  vol.~5, no.~4, pp. 297--323, 1992.

\bibitem{markowitz1952portfolio}
H.~Markowitz, ``Portfolio selection,'' \emph{The Journal of Finance}, vol.~7,
  no.~1, pp. 77--91, 1952.

\bibitem{tamar2012policy}
A.~Tamar, D.~Di~Castro, and S.~Mannor, ``Policy gradients with variance related
  risk criteria,'' in \emph{International Coference on International Conference
  on Machine Learning}, 2012, pp. 1651--1658.

\bibitem{mannor2013algorithmic}
S.~Mannor and J.~N. Tsitsiklis, ``Algorithmic aspects of mean--variance
  optimization in {M}arkov decision processes,'' \emph{European Journal of
  Operational Research}, vol. 231, no.~3, pp. 645--653, 2013.

\bibitem{howard1972risk}
R.~A. Howard and J.~E. Matheson, ``Risk-sensitive {M}arkov decision
  processes,'' \emph{Management Science}, vol.~18, no.~7, pp. 356--369, 1972.

\bibitem{whittle1990risk}
P.~Whittle, \emph{Risk-sensitive optimal control}.\hskip 1em plus 0.5em minus
  0.4em\relax Wiley, 1990.

\bibitem{borkar2002q}
V.~S. Borkar, ``Q-learning for risk-sensitive control,'' \emph{Mathematics of
  Operations Research}, vol.~27, no.~2, pp. 294--311, 2002.

\bibitem{rockafellar2002conditional}
R.~T. Rockafellar and S.~Uryasev, ``Conditional value-at-risk for general loss
  distributions,'' \emph{Journal of banking \& finance}, vol.~26, no.~7, pp.
  1443--1471, 2002.

\bibitem{prashanth2018risk}
L.~A. Prashanth and M.~Fu, ``Risk-sensitive reinforcement learning: {A}
  constrained optimization viewpoint,'' \emph{arXiv:1810.09126}, 2018.

\bibitem{prashanth2016cumulative}
L.~A. Prashanth, C.~Jie, M.~Fu, S.~Marcus, and C.~Szepesv{\'a}ri, ``Cumulative
  prospect theory meets reinforcement learning: {P}rediction and control,'' in
  \emph{International Conference on Machine Learning}, 2016.

\bibitem{jie2018stochastic}
C.~Jie, L.~A. Prashanth, M.~Fu, S.~Marcus, and C.~Szepesv{\'a}ri, ``Stochastic
  optimization in a cumulative prospect theory framework,'' \emph{IEEE
  Transactions on Automatic Control}, vol.~63, no.~9, pp. 2867--2882, 2018.

\bibitem{lin2013dynamic}
K.~Lin and S.~I. Marcus, ``Dynamic programming with non-convex risk-sensitive
  measures,'' in \emph{American Control Conference}.\hskip 1em plus 0.5em minus
  0.4em\relax IEEE, 2013, pp. 6778--6783.

\bibitem{lin2018probabilistically}
K.~Lin, C.~Jie, and S.~I. Marcus, ``Probabilistically distorted risk-sensitive
  infinite-horizon dynamic programming,'' \emph{Automatica}, vol.~97, pp. 1--6,
  2018.

\bibitem{cubuktepe2018verification}
M.~Cubuktepe and U.~Topcu, ``Verification of {M}arkov decision processes with
  risk-sensitive measures,'' in \emph{Annual American Control Conference
  (ACC)}.\hskip 1em plus 0.5em minus 0.4em\relax IEEE, 2018, pp. 2371--2377.

\bibitem{lipp2016variations}
T.~Lipp and S.~Boyd, ``Variations and extension of the convex--concave
  procedure,'' \emph{Optimization and Engineering}, vol.~17, no.~2, pp.
  263--287, 2016.

\bibitem{majumdar2020should}
A.~Majumdar and M.~Pavone, ``How should a robot assess risk? towards an
  axiomatic theory of risk in robotics,'' in \emph{Robotics Research}.\hskip
  1em plus 0.5em minus 0.4em\relax Springer, 2020, pp. 75--84.

\bibitem{barberis2013thirty}
N.~C. Barberis, ``Thirty years of prospect theory in economics: {A} review and
  assessment,'' \emph{Journal of Economic Perspectives}, vol.~27, no.~1, pp.
  173--96, 2013.

\bibitem{prelec1998probability}
D.~Prelec, ``The probability weighting function,'' \emph{Econometrica}, pp.
  497--527, 1998.

\bibitem{jaakkola1994convergence}
T.~Jaakkola, M.~I. Jordan, and S.~P. Singh, ``On the convergence of stochastic
  iterative dynamic programming algorithms,'' \emph{Neural Computation},
  vol.~6, no.~6, pp. 1185--1201, 1994.

\bibitem{van2009theoretical}
H.~Van~Seijen, H.~Van~Hasselt, S.~Whiteson, and M.~Wiering, ``A theoretical and
  empirical analysis of expected {SARSA},'' in \emph{IEEE Symposium on Adaptive
  Dynamic Programming and Reinforcement Learning}, 2009, pp. 177--184.

\bibitem{borkar2000ode}
V.~S. Borkar and S.~P. Meyn, ``The {ODE} method for convergence of stochastic
  approximation and reinforcement learning,'' \emph{SIAM Journal on Control and
  Optimization}, vol.~38, no.~2, pp. 447--469, 2000.

\end{thebibliography}
\end{document}